\documentclass{article}

\usepackage{arxiv}

\usepackage[utf8]{inputenc} 
\usepackage[T1]{fontenc}    
\usepackage{hyperref}       
\usepackage{url}            
\usepackage{booktabs}       
\usepackage{amsfonts}       
\usepackage{nicefrac}       
\usepackage{microtype}      
\usepackage{lipsum}		
\usepackage{amsmath,amsfonts,amssymb,amsthm,amsbsy}
\usepackage{graphicx, float}
\usepackage[round]{natbib}

\newcommand{\bs}{\boldsymbol}
\newcommand{\sgn}{\mathrm{sgn}}
\newcommand\norm[1]{\left\lVert#1\right\rVert}

\newtheorem{thm}{Theorem}
\newtheorem{lem}{Lemma}

\newtheorem{defi}{Definition}

\newcommand{\be}{\begin{equation}}
\newcommand{\ee}{\end{equation}}
\newcommand{\ba}{\begin{array}}
\newcommand{\ea}{\end{array}}

\begin{document}

%

%

\title{Understanding and Quantifying Adversarial Examples Existence in Linear Classification}


\author{
  Xupeng Shi \\
  Department of Mathematics\\
  Northeastern University\\
  Boston, USA \\
  \texttt{shi.xup@husky.neu.edu} \\
   \And
 A. Adam Ding \\
  Department of Mathematics\\
  Northeastern University\\
  Boston, USA \\
  \texttt{a.ding@northeastern.edu} \\
}

\maketitle


\begin{abstract}
State-of-art deep neural networks (DNN) are vulnerable to attacks by adversarial examples: a carefully designed small perturbation to the input, that is imperceptible to human, can mislead DNN.
To understand the root cause of adversarial examples, we quantify the probability of adversarial example existence for linear classifiers.
Previous mathematical definition of adversarial examples only involves the overall perturbation amount, and we propose a more practical relevant definition of strong adversarial examples that separately limits the perturbation along the signal direction also.
We show that linear classifiers can be made robust to strong adversarial examples attack in cases where no adversarial robust linear classifiers exist under the previous definition.
The quantitative formulas are confirmed by numerical experiments using a linear support vector machine (SVM) classifier.
The results suggest that designing general strong-adversarial-robust learning systems is feasible but only through incorporating human knowledge of the underlying classification problem.
\end{abstract}

\section{Introduction}

The deep neural networks (DNN) are widely used as the state-of-art machining learning classification systems due to its great performance gains in recent years. Meanwhile adversarial examples, first pointed out by \cite{Szegedy}, emerges as a novel peculiar security threat against such systems: a small perturbation that is unnoticeable to human eyes can cause the DNNs to misclassify.
Various adversarial algorithms have since been developed to efficiently find adversarial examples~\citep{ExplainAd,Deepfool,CW,DLResistance}.
The adversarial examples have also been demonstrated to misled DNN based classification systems in physical world applications~\citep{faceAd, advpatch, physicaladv, synadv}.
Various defense methods have also been proposed to prevent adversarial example attacks: Adversarial training \citep{Szegedy, ExplainAd}; Defensive distillation \cite{distillation}; Minmax robust training~\citep{DLResistance, DistributionalRobustness}; Input transformation \cite{FeatureSqueezing}.
However, many of the defenses are shown to be vulnerable to attacks taking such defense strategies into consideration~\citep{athalye2018obfuscated}.

Recently, \cite{AdInevitable} showed that, for two classes of data distributed with bounded probability densities on a compact region of a high dimensional space, no classifier can both have low misclassification rate and be robust to adversarial examples attack. So are we left hopeless against such threat? Theoretical analysis for understanding adversarial examples is needed to address this issue.
\cite{ExplainAd,AnalysisRobustness} pointed out that susceptibility of DNN classifiers to adversarial attacks could be related to their locally linear behaviours.
The existence of adversarial examples is not unique to DNN, traditional linear classifiers also have adversarial examples.
In this paper, we extend the understanding of adversarial examples by quantifying the probability of their existence for a simple case of linear classifiers that performs binary classification on Gaussian mixture data.

In previous literature, a data point $\bs{x}$ is mathematically defined as having an adversarial example $\bs{x}'=\bs{x}+\bs{v}$ when the perturbation amount $\norm{\bs{v}}$ is small and $\bs{x}'$  is classified differently from $\bs{x}$.
This definition does not exclude genuine signal perturbation.
For example, if a dog image $\bs{x}$ is perturbed to an image $\bs{x}'$ that is classified as a cat by both human and the machine classifier, then $\bs{x}'$ should not be an adversarial example even if  $\norm{\bs{v}}=\norm{\bs{x}'-\bs{x}}$ is small.
The proper definition needs to capture the novelty of adversarial examples attack: while a human would consider two images $\bs{x}'$ and $\bs{x}$ very similar and consider both clearly as dogs, a machine classifier misclassifies $\bs{x}'$ as a cat.
While defining genuine signal perturbation for general learning problems is difficult mathematically, the signal perturbation is clear in the binary linear classification for Gaussian mixture data. We therefore propose a new definition of strong-adversarial examples that limits the perturbation amount in the signal direction separately from the limit on overall perturbation amount.

In this paper, we derive quantitative formulas for the probabilities of adversarial and strong-adversarial examples existence in the binary linear classification problem.
Our quantitative analysis shows that an adversarial-robust linear classifier requires much higher signal-to-noise ratio (SNR) in data than a good performing classifier does.
Therefore, in many practical applications, adversarial-robust classifiers may not be available nor are such classifiers desirable.
On the contrary, useful strong-adversarial-robust linear classifiers exists at the SNR similar to that required by the existence of any useful linear classifiers, however, they require better designed training algorithms.

The paper is organized as follows. Section~\ref{sec:theo} presents the notations and definitions of (strong-)adversarial examples and derive explicit formulas for the probability of their existence.
Section~\ref{sec:exp} presents numerical experiments. The formulas are confirmed experimentally, and then are used to illustrate their implication on the vulnerability against (strong-)adversarial example attacks. 
Section~\ref{sec:re} discusses how our results relate to some works in literature and summarize their implication on general adversarial attack defenses.

\section{Adversarial Rates Analysis of Linear Binary Classifier on Gaussian Mixture Data}
\label{sec:theo}

We first introduce our definitions of adversarial and strong-adversarial examples, and then we characterize their existence through defining sets.
Using the defining sets, we derive explicit probability rates of (strong-)adversarial examples existence for linear classifiers on Gaussian mixture data.

\subsection{Definition of Adversarial and Strong-Adversarial Examples}\label{sec:def}

The classical adversarial examples are defined as follows:

\begin{defi}\label{def:adv}\footnote{We don't distinguish the targeted and untargeted adversarial examples here because for binary classification they are the same.}
Given a classifier $C$, an $\varepsilon$-adversarial example of a data vector $\bs{x}$  is another data vector $\bs{x}'$ such that $\norm{\bs{x}- \bs{x}'} \leq \varepsilon$ but $C(\bs{x})\neq C(\bs{x}')$.
\end{defi}
Without loss of generality, in this paper we focus on $\ell_2$ norm perturbations. If not specified, $\norm{\cdot}$ in the following refers to the $\ell_2$ norm. The general $\ell_p$ norm ($p\ge 1$) perturbation is studied in the Appendix~\ref{sec:appendix}, and the results will be stated in the discussion section.

For a general machine classification problem, it is reasonable to only consider adversarial examples since the signal direction is often not easily definable mathematically. Here we consider the simple binary linear classification of Gaussian mixture data where the signal direction can be clearly distinguished. For two classes labeled `$+$' and `$-$' respectively, a linear classifier is  $C(\bs{x};\bs{w},b)= \{\bs{w}\cdot\bs{x}+b >0\}$ where `$\cdot$' denotes the inner product of two vectors. Here the parameters $\bs{w}$ and $b$ are respectively the weight vector and the bias term. For the classical Gaussian mixture data problem, for each of the two classes, the $d$-dimensional data vector $\bs{x}$ comes from a multivariate Gaussian distribution $N(\bs{\mu}_i,\sigma_i^2\bs{I}_d)$, $i=$ `$+$' or `$-$'. Notice the optimal ideal classifier here is the Bayes classifier 
$C(\bs{x};\bs{\mu},\bar{\bs{\mu}})=\{\bs{\mu}\cdot(\bs{x}-\bar{\bs{\mu}})>0\}$\footnote{Here we just use the optimal Bayes classfier for balanced case since we are focusing on the balanced case in the following text.} where $\bs{\mu}=\frac{1}{2}(\bs{\mu}_+-\bs{\mu}_-),\bar{\bs{\mu}}=\frac{1}{2}(\bs{\mu}_{+}+\bs{\mu}_{-})$.

For this problem, the data distributions of the two classes only differ in their means $\bs{\mu}_+$ and $\bs{\mu}_-$.
Thus the signal direction is $\bs{\mu}_0=\bs{\mu}/\norm{\bs{\mu}}$.
Adding $2\norm{\bs{\mu}}$ amount of perturbation along the signal direction changes the `$-$' class data distribution to the `$+$' class data distribution exactly, rending all classifiers unable to defend against such a perturbation.

In previous literature, the adversarial examples definition does not limit perturbation along the signal direction, therefore we propose a new definition that limits the perturbation along the signal direction separately by an amount $\delta$, we will refer these examples as \textit{strong-adversarial examples} .
\begin{defi}\label{def:s-adv}
Given a classifier $C$, an $(\varepsilon,\delta)$-strong-adversarial example of a data vector $\bs{x}$  is another data vector $\bs{x}'$ such that $\norm{\bs{x}- \bs{x}'} \leq \varepsilon$ and $|(\bs{x}-\bs{x}')\cdot\bs{\mu}_0|\leq\delta$ but $C(\bs{x})\neq C(\bs{x}')$.
\end{defi}

To illustrate the difference between the adversarial examples and the strong-adversarial examples, we consider the following examples visualized in Figure~\ref{fig:example}. Here, Figure~\ref{fig:example}(a) shows a data vector $\bs{x}$ of dimension $d=19 \times 19 = 361$ from the `$+$' class. To visualize, each component of the data vector is mapped onto $[0,1]$ via function $\frac{1}{2}(\tanh \frac{2x}{3}+1)$ and then displayed in grey scale as a $19\times 19$ image~\citep{CW}.
\begin{figure*}[htbp]
\begin{center}
\includegraphics[width=\textwidth]{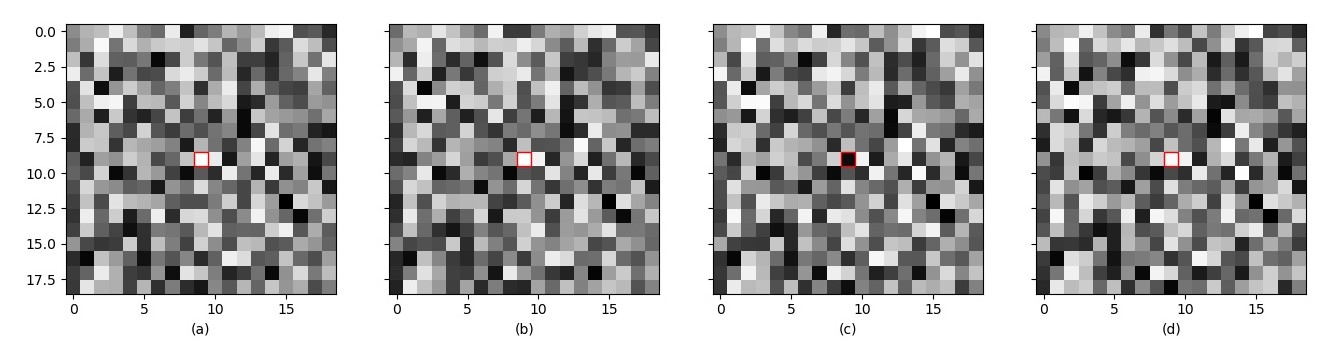}
\caption{(a) a data point $\bs{x}$ from the `$+$' class; 
(b) a randomly perturbed $\bs{x}'$; 
(c) an adversarial $\bs{x}'$ but not strong-adversarial; 
(d) a strong-adversarial $\bs{x}'$. 
All three perturbations are of the same amount $\varepsilon = 5.7$ and $\norm{\bs{\mu}}=4$.
The center grid cell within the red boundary contains the real class signal.
}\label{fig:example}
\end{center}
\end{figure*}

The two means $\bs{\mu}_+$ and $\bs{\mu}_-$ are chosen to be zero at every component of the vector except the component corresponding to center grid cell (shown with red boundary in~Figure~\ref{fig:example}). Hence the optimal Bayes classifier identifies the image as from `$+$' (or `$-$') class when the center grid cell within the red boundary appears to be white (or black). With a perturbation amount of  $\varepsilon=0.3\times 19 = 5.7$, Figure~\ref{fig:example}(b) shows a randomly perturbed $\bs{x}'$ which is hardly distinguishable from the first image $\bs{x}$ to the human eye. This confirms that, in defending against realistic threats, $\varepsilon$ of magnitude $O(\sqrt{d})$ needs to be studied. (Detailed discussion of $\varepsilon$ order is in subsection~\ref{sec:rate}.)

For a trained support vector machine (SVM) classifier, Figure~\ref{fig:example}(c) and (d) shows two adversarial examples with the same $\varepsilon= 5.7$, but only the last one in (d) is strong-adversarial for $\delta=1.2$. (Section~\ref{sec:exp} provides detailed setup of this experiment.) The adversarial attacks present a novel threat: a machine classifier misclassifies the perturbed data points that a human would not have noted the difference. We can see that our strong-adversarial example definition focus attention on this novel threat. In contrast, under the traditional definition, the adversarial examples include examples similar to Figure~\ref{fig:example}(c) that would indeed be classified by human into another class. We now quantitatively analyze the existence of adversarial and strong-adversarial examples.

\subsection{The Defining Sets}\label{sec:set}

Here we characterize the defining sets where the (strong-)adversarial examples exist. Then we quantify the probability of data falling into these defining sets in the next subsection~\ref{sec:rate}.

We denote $\Omega_\varepsilon=\{\bs{x}: \bs{x}\ \mbox{has\ an\ }\varepsilon\mbox{-adversarial\ example}\}$ and $\Omega_{\varepsilon,\delta}=\{\bs{x}: \bs{x}\ \mbox{has\ an\ }(\varepsilon,\delta)\mbox{-strong-adversarial}$ example\}. Furthermore, for a fixed perturbation $\bs{v}$, we denote the set where $\bs{v}$ changes classification as $\Omega(\bs{v})=\{\bs{x}\in\mathbb{R}^d: C(\bs{x}+\bs{v})\neq C(\bs{x})\}$.

For any data point $\bs{x}$ in $\Omega_\varepsilon$, there exists a $\bs{v}$ with $\norm{\bs{v}}\leq\varepsilon$ such that $\bs{x}+\bs{v}$ is classified differently from $\bs{x}$. In other words, the distance of $\bs{x}$ from the classifier's decision boundary is less than $\varepsilon$. For a linear classifier $C(\bs{x};\bs{w},b)= \{\bs{w}\cdot\bs{x}+b >0\}$, the normal direction of its decision boundary is $\bs{v}_0={\bs{w}}/\norm{\bs{w}}$. Thus, perturbing $\bs{x}$ by $\varepsilon$ amount along one of the two directions $\bs{v}_0$ or $-\bs{v}_0$ will cross the linear decision boundary. That is, $\Omega_\varepsilon \subseteq \Omega(\varepsilon \bs{v}_0)\cup \Omega(-\varepsilon\bs{v}_0)$. Since it is obvious from the definition that $\Omega_\varepsilon=\bigcup_{\norm{\bs{v}}\leq\varepsilon}\Omega(\bs{v}) \supseteq \Omega(\varepsilon \bs{v}_0)\cup \Omega(-\varepsilon\bs{v}_0)$, we have $\Omega_\varepsilon = \Omega(\varepsilon \bs{v}_0)\cup \Omega(-\varepsilon\bs{v}_0)$. In summary, to judge if $\bs{x} \in \Omega_\varepsilon$, we only need to check the perturbation along the normal direction $\bs{v}_0$.

In contrast, our definition of strong-adversarial examples only allows $\delta$ amount of perturbation along the signal notation $\bs{\mu}_0$, hence it is not sufficient to only check perturbations $\varepsilon\bs{v}_0$ and $-\varepsilon\bs{v}_0$ for judging if $\bs{x} \in \Omega_{\varepsilon,\delta}$.
Let $\theta$ denote the deflected angle between $\bs{\mu}_0$ and $\bs{v}_0$. (Without loss of generality, we choose the $\theta$ value such that $0 \le \theta \le \pi/2$.)
Then we can decompose $\bs{v}_0$ into two components along and orthogonal to the signal direction $\bs{\mu}_0$ respectively.
That is, $\bs{v}_0= \cos\theta \bs{\mu}_0+ \sin \theta \bs{n}_0$ where $\bs{n}={\bs{v}_0-(\bs{v}_0\cdot \bs{\mu}_0)\bs{\mu}_0}$ and $\bs{n}_0=\bs{n}/\norm{\bs{n}}$.
When $\varepsilon\cos\theta \le \delta$, the adversarial example resulting from the $\varepsilon\bs{v}_0$ perturbation is also strong-adversarial by definition.
When $\varepsilon\cos\theta > \delta$, however, $\varepsilon\bs{v}_0$ is no longer an allowable perturbation in the strong-adversarial example definition.
Then we need to check whether classification change is caused by a perturbation of $\delta$ amount along $\bs{\mu}_0$ direction and $\sqrt{\varepsilon^2-\delta^2}$ amount along $\bs{n}_0$ direction.
That is, , to judge if $\bs{x} \in \Omega_{\varepsilon,\delta}$, we need to check perturbations $\bs{u}_2=\delta \bs{\mu}_0+\sqrt{\varepsilon^2-\delta^2}\bs{n}_0$ and $-\bs{u}_2$. We summarize the defining sets characterization in the following lemma whose detailed proof is in the Appendix~\ref{sec:app}.


\begin{lem}\label{lem:set} The defining sets for $\varepsilon$-adversarial and $(\varepsilon,\delta)$-strong-adversarial examples are given by:
\begin{equation}\label{eq:set}
\Omega_\varepsilon = \Omega(\varepsilon \bs{v}_0)\cup \Omega(-\varepsilon\bs{v}_0); \quad \Omega_{\varepsilon,\delta} = \Omega(\bs{u}_2)\cup\Omega(-\bs{u}_2)
\end{equation}
where $\bs{u}_2=\beta\bs{\mu}_0+\sqrt{\varepsilon^2-\beta^2}\bs{n}_0, \beta=\min(\varepsilon\cos\theta,\delta)$.
\end{lem}
Next, we use these defining sets to quantify the probabilities of (strong-)adversarial example existence.

\subsection{Adversarial and Strong-Adversarial Rates}\label{sec:rate}

For the binary classification problem, a random data vector comes from
the Gaussian mixture distribution $p(\bs{x})=\lambda_+\varphi_{+}(\bs{x})+\lambda_-\varphi_{-}(\bs{x})$, where $\varphi_i(\bs{x})$ is the probability density function of the multivariate Gaussian $N(\bs{\mu}_i,\sigma_i^2\bs{I}_d)$ and $\lambda_i$ is the probability that the data vector belongs to the class of $i=$ `$+$' or `$-$'.
For simplicity, we focus on the balanced classes case of $\lambda_{+}=\lambda_{-}=0.5$ and also $\sigma_+=\sigma_-=\sigma$.

\paragraph{Adversarial Rate} For a random data vector $\bs{x}$ from the `$+$' class, it has an $\varepsilon$-adversarial example $\bs{x}'$ if it is classified correctly by $\bs{w}\cdot\bs{x}+b>0$ and $\bs{x} \in \Omega(-\varepsilon\bs{v}_0)$. Thus the adversarial rate from the `$+$' class is
\begin{align}
\lambda_{+} pr[\bs{w}\cdot \bs{x}+b>0,\bs{w}\cdot(\bs{x}-\varepsilon\bs{v}_0)+b<0 \ |\varphi_+(\bs{x})]  = 0.5 pr[0<\bs{w}\cdot \bs{x}+b<\varepsilon\norm{\bs{w}} \ |\varphi_+(\bs{x})] .
\end{align}
Since under the multivariate Gaussian $N(\bs{\mu}_+,\sigma^2\bs{I}_d)$ distribution $\varphi_+(\bs{x})$,  $\bs{w}\cdot\bs{x}+b$ is a univariate Gaussian random variable with mean $\bs{w}\cdot\bs{\mu}_+ +b$ and variance $\norm{\bs{w}}^2\sigma^2$, the above quantity becomes
\be\label{eq:adv+}
0.5\bigg[\Phi\bigg(\frac{\varepsilon\norm{\bs{w}}-(\bs{w}\cdot\bs{\mu}_++b)}{\norm{\bs{w}}\sigma}\bigg) - \Phi\bigg(\frac{-(\bs{w}\cdot\bs{\mu}_++b)}{\norm{\bs{w}}\sigma}\bigg)\bigg].
\ee
Here $\Phi(\cdot)$ denotes the cumulative distribution function (CDF) of the standard Gaussian distribution $N(0,1)$.
Similarly, the adversarial rate from the `$-$' class is
\begin{align}\label{eq:adv-}
\lambda_- pr[-\varepsilon\norm{\bs{w}}<\bs{w}\cdot\bs{x}+b<0|\varphi_-(\bs{x})]
=0.5\bigg[\Phi\bigg(\frac{-(\bs{w}\cdot\bs{\mu}_-+b)}{\norm{\bs{w}}\sigma}\bigg) - \Phi\bigg(\frac{-\varepsilon\norm{\bs{w}}-(\bs{w}\cdot\bs{\mu}_-+b)}{\norm{\bs{w}}\sigma}\bigg)\bigg].
\end{align}
Recall $\bs{\mu}=\frac{1}{2}(\bs{\mu}_+-\bs{\mu}_-),\bar{\bs{\mu}}=\frac{1}{2}(\bs{\mu}_{+}+\bs{\mu}_{-})$. If we denote  $b'=\bs{w}\cdot\bar{\bs{\mu}}+b$, then we can rewritten the expressions as $\bs{w}\cdot\bs{\mu}_{\pm}+b=\pm \bs{w}\cdot\bs{\mu}+b'$. Combining equations \eqref{eq:adv+} and \eqref{eq:adv-}, we have the overall adversarial rate as
\be\label{eq:adv-rate'}
\ba{rl}
p_{adv}
&= 0.5\bigg[\Phi\bigg(\frac{\varepsilon}{\sigma}-\frac{\bs{w}\cdot\bs{\mu}+b'}{\norm{\bs{w}}\sigma}\bigg) - \Phi\bigg(-\frac{\bs{w}\cdot\bs{\mu}+b'}{\norm{\bs{w}}\sigma}\bigg) +  \Phi\bigg(\frac{\bs{w}\cdot\bs{\mu}-b'}{\norm{\bs{w}}\sigma}\bigg)
-  \Phi\bigg(-\frac{\varepsilon}{\sigma} + \frac{\bs{w}\cdot\bs{\mu}-b'}{\norm{\bs{w}}\sigma}\bigg) \bigg] \\
 &= 0.5\bigg[  \Phi\bigg(\frac{\bs{w}\cdot\bs{\mu}+b'}{\norm{\bs{w}}\sigma}\bigg)
-\Phi\bigg(\frac{\bs{w}\cdot\bs{\mu}+b'}{\norm{\bs{w}}\sigma} - \frac{\varepsilon}{\sigma}\bigg)
+ \Phi\bigg(\frac{\bs{w}\cdot\bs{\mu}-b'}{\norm{\bs{w}}\sigma}\bigg)
-  \Phi\bigg(\frac{\bs{w}\cdot\bs{\mu}-b'}{\norm{\bs{w}}\sigma}-\frac{\varepsilon}{\sigma}\bigg) \bigg]
\ea
\ee

Also notice that the misclassification rates from the two classes are respectively $\lambda_+\Phi[ - (\bs{w}\cdot\bs{\mu}_++b)/(\norm{\bs{w}}\sigma)] = 0.5 \{1- \Phi[ (\bs{w}\cdot\bs{\mu}+b')/(\norm{\bs{w}}\sigma)]\}$ and $\lambda_-\{1-\Phi[ - (\bs{w}\cdot\bs{\mu}_-+b)/(\norm{\bs{w}}\sigma)]\} = 0.5 \{1-\Phi[(\bs{w}\cdot\bs{\mu}-b')/(\norm{\bs{w}}\sigma)]\}$. Thus the overall misclassification rate is
\be\label{eq:pm}
\ba{cl}
p_m & = 1 - 0.5\bigg[\Phi\bigg(\frac{\bs{w}\cdot\bs{\mu}+b'}{\norm{\bs{w}}\sigma}\bigg) + \Phi\bigg(\frac{\bs{w}\cdot\bs{\mu}-b'}{\norm{\bs{w}}\sigma}\bigg) \bigg].
\ea
\ee
We combine equations \eqref{eq:adv-rate'} and \eqref{eq:pm} into the following Theorem.
\begin{thm}\label{thm:adv-rate}
The overall adversarial rate of a linear classifier for the balanced Gaussian mixture data is
\be
\ba{cl}
p_{adv}=1-p_m-0.5\bigg[\Phi\bigg(\frac{\bs{w}\cdot\bs{\mu}+b'}{\norm{\bs{w}}\sigma}-\frac{\varepsilon}{\sigma}\bigg) 
+\Phi\bigg(\frac{\bs{w}\cdot\bs{\mu}-b'}{\norm{\bs{w}}\sigma}-\frac{\varepsilon}{\sigma}\bigg)\bigg].
\ea
\label{eq:adv-rate}
\ee
\end{thm}

To be robust against adversarial attacks, a linear classifier needs a low adversarial rate.
For the classifier to be useful, it also needs a low misclassification rate. Hence we should look at the sum of misclassification rate and adversarial rate, which we call the {\it adversarial-error} rate:
\be
\ba{cl}
p_{err} = p_{adv} + p_m &= 1-0.5\bigg[\Phi\bigg(\frac{\bs{w}\cdot\bs{\mu}+b'}{\norm{\bs{w}}\sigma}-\frac{\varepsilon}{\sigma}\bigg) 
 +\Phi\bigg(\frac{\bs{w}\cdot\bs{\mu}-b'}{\norm{\bs{w}}\sigma}-\frac{\varepsilon}{\sigma}\bigg)\bigg]
\ea
\label{eq:adv-err}
\ee
Comparing equation \eqref{eq:adv-err} with \eqref{eq:pm}, we can see why adversarial-robustness is hard to achieve.

First, the misclassification rate $p_m$ in \eqref{eq:pm} is minimized by the Bayes classifier with $b'=0$ and 
$\bs{w}\cdot\bs{\mu}=\norm{\bs{w}}\norm{\bs{\mu}}$. Hence the best $p_m$ value is $1-\Phi(\norm{\bs{\mu}}/\sigma)$.
There exists useful classifiers when $\norm{\bs{\mu}}/\sigma$ is big enough to make $1-\Phi(\norm{\bs{\mu}}/\sigma)$ small. This is achieved for $\norm{\bs{\mu}}/\sigma =O(1)$. For example, when $\norm{\bs{\mu}}/\sigma =3$, the misclassification rate of the Bayes classifier is around $0.1\%$.

However, to achieve a low adversarial-error rate in \eqref{eq:adv-err}, the required SNR $\norm{\bs{\mu}}/\sigma$ can be much bigger. When $\bs{w}\cdot\bs{\mu} > \varepsilon\norm{\bs{w}}$, a lower bound for the adversarial-error rate is
\be\label{eq:adv-err-bound}
p_{err} \ge 1 - \Phi\bigg(\frac{\bs{w}\cdot\bs{\mu}}{\norm{\bs{w}}\sigma}-\frac{\varepsilon}{\sigma}\bigg) \ge 1 - \Phi\bigg(\frac{\norm{\bs{\mu}}}{\sigma}-\frac{\varepsilon}{\sigma}\bigg).
\ee
Therefore, the existence of a useful adversarial-robust linear classifier requires $\norm{\bs{\mu}}/\sigma -{\varepsilon}/{\sigma}=O(1)$ instead. Notice that, for this Gaussian mixture data setup, the noise  in each class follows the $N(\bs{0},\sigma^2\bs{I}_d)$ distribution with an expected square of $\ell_2$ norm of $d \sigma^2$.
Therefore, for a positive constant value $\eta_a<1$, the perturbation amount of $\varepsilon=\eta_a \sqrt{d} \sigma$ is smaller than the average noise in data and generally is hard to detect.
Hence, for the typical high-dimensional data applications, an adversarial-robust linear classifier needs to protect against perturbation amount of $\varepsilon= O(\sqrt{d})$ which implies that $\norm{\bs{\mu}}/\sigma = O(\sqrt{d})$ is needed from equation \eqref{eq:adv-err-bound}. Next, we show that this high SNR requirement is not needed for a strong-adversarial-robust linear classifier.

\paragraph{Strong-Adversarial Rate}
The derivation of the strong-adversarial rate is very similar to that of the adversarial rate.
From equation~\eqref{eq:set}, the difference between the adversarial defining set and the strong-adversarial defining set is only that $\varepsilon \bs{v}_0$ is replaced by $\bs{u}_2=\beta\bs{\mu}_0+\sqrt{\varepsilon^2-\beta^2}\bs{n}_0$.
Hence the strong-adversarial rate from the '$+$' class is
$$
\ 0.5 pr[0<\bs{w}\cdot \bs{x}+b< \bs{w} \cdot \bs{u}_2 |\varphi_+(\bs{x})] .
$$
Since $\bs{w}\cdot\bs{\mu}_0=\norm{\bs{w}}\cos\theta$ and $\bs{w}\cdot\bs{n}_0=\norm{\bs{w}}\sin\theta$, we have $\bs{w} \cdot \bs{u}_2 = (\beta\cos\theta + \sqrt{\varepsilon^2-\beta^2} \sin\theta ) \norm{\bs{w}}$ where $\beta=\min(\varepsilon\cos\theta,\delta)$. We denote
\be\label{eq:fun_g}
g(\varepsilon, \delta, \theta) = \beta\cos\theta + \sqrt{\varepsilon^2-\beta^2} \sin\theta.
\ee
Thus replacing $\varepsilon\norm{\bs{w}}$ by $g(\varepsilon, \delta, \theta) \norm{\bs{w}}$ in equations from \eqref{eq:adv+} to \eqref{eq:adv-err}, we have the following Theorem.
\begin{thm}\label{thm:s-adv-rate}
The overall strong adversarial rate and {\it strong-adversarial-error} rate of a linear classifier are
\begin{align}
p_{s-adv} &=1-p_m-0.5\bigg[\Phi\bigg(\frac{\bs{w}\cdot\bs{\mu}+b'}{\norm{\bs{w}}\sigma}-\frac{g(\varepsilon, \delta, \theta)}{\sigma} \bigg) 
+\Phi\bigg(\frac{\bs{w}\cdot\bs{\mu}-b'}{\norm{\bs{w}}\sigma}-\frac{g(\varepsilon, \delta, \theta)}{\sigma}\bigg)\bigg]\label{eq:s-adv-rate}\\
p_{s-err} &= p_{s-adv} + p_m 
 = 1-0.5\bigg[\Phi\bigg(\frac{\bs{w}\cdot\bs{\mu}+b'}{\norm{\bs{w}}\sigma}-\frac{g(\varepsilon, \delta, \theta)}{\sigma} \bigg) 
 +\Phi\bigg(\frac{\bs{w}\cdot\bs{\mu}-b'}{\norm{\bs{w}}\sigma}-\frac{g(\varepsilon, \delta, \theta)}{\sigma}\bigg)\bigg]\label{eq:s-err}
\end{align}
\end{thm}

Compared to the analysis above, the existence of a useful strong-adversarial-robust linear classifier requires $\norm{\bs{\mu}}/\sigma -g(\varepsilon, \delta, \theta)/{\sigma}=O(1)$ instead.
Besides the overall perturbation amount $\varepsilon$, the function $g(\varepsilon, \delta, \theta)$ in equation~\eqref{eq:fun_g} is also affected by two other factors: the signal direction perturbation amount $\delta$ and the angle $\theta$ between the classifier and the ideal Bayes classifier. What is the practical relevant amount $\delta$ we should study? Let $\delta=\eta_s \mu = \eta_s \norm{\bs{\mu}}$. When $\eta_s>1$, a $\delta$ amount perturbation along the signal direction to all '$+$' class data points will make more than half of them be classified as '$-$' by the Bayes classifier (also to human eye, e.g., Figure~\ref{fig:example}(c)). Therefore, when studying real strong-adversarial perturbations (imperceptible to human  but confuses machine) mathematically, we need to focus on $\eta_s<1$. That is, $ \delta = O(1)$. Compared to the overall perturbation amount $\varepsilon= O(\sqrt{d})$ discussed earlier, we see that $\delta\ll\varepsilon$ for typical high-dimensional data applications. When $\delta\ll\varepsilon$, $g(\varepsilon, \delta, \theta) \approx \delta \cos\theta + \varepsilon \sin \theta$. Hence if the linear classifier is well-trained to have small $\theta$ and small bias $b'$ (i.e., very close to the Bayes classifier), then its strong-adversarial-error rate is approximately $1 - \Phi[(1-\eta_s){\norm{\bs{\mu}}}/{\sigma}]$,
which can be made small when SNR $\norm{\bs{\mu}}/\sigma$ is of order $O(1)$. That is, with good training, we can find a useful strong-adversarial-robust linear classifier when $\norm{\bs{\mu}}/\sigma=O(1)$. In contrast, no training can make the linear classifier to be useful and adversarial-robust unless the SNR $\norm{\bs{\mu}}/\sigma$ is much bigger, at the order of $O(\sqrt{d})$.

The conclusion for the analysis using $\ell_p$ norm (see Appendix~\ref{sec:appendix} for details) is similar.
There exists a useful strong-adversarial-robust linear classifier for constant order SNR $\norm{\bs{\mu}}/\sigma=O(1)$, but a useful $\ell_p$-adversarial-robust linear classifier only exists when SNR is much bigger, at the order of $O(d^{min(1/p,1/2)})$.

\section{Numerical Studies and Analysis of Adversarial Examples}
\label{sec:exp}


\subsection{(Strong-)Adversarial Rates for the Linear SVM}\label{sec:num}

\paragraph{Settings}
We first conduct numerical experiments of a support vector machine (SVM) classifier on the Gaussian mixture data.
We randomly generate 5000 data points from the balanced mixture distribution $0.5N(\bs{\mu}_+,\sigma^2\bs{I}_d)+0.5N(\bs{\mu}_-,\sigma^2\bs{I}_d)$, and randomly splits them into $4000$ train data and $1000$ test data.
We set $\bs{\mu}_+=-\bs{\mu}_-=[\mu,0,\cdots,0]$, $\sigma=1$ and $d=19\times 19$.
A linear SVM is trained on the training data using the python {\it scikit-learn} package and its default setting.
Then for each test data vector, we check if it has any adversarial and strong-adversarial example, for $\varepsilon=\eta_a \sqrt{d} \sigma=19 \eta_a$ and $\delta=\eta_s \mu$.
Figure~\ref{fig:example} earlier visualizes one such test data vector and its adversarial and strong-adversarial examples for $\eta_a=\eta_s=0.3$.
We conduct this experiment for various values of $\eta=\eta_a=\eta_s$ and $\mu$, and for each parameter combination, the simulation is repeated  $1000$ times.
Figure~\ref{fig:figure1} plots three empirical rates (misclassification, adversarial-error and strong-adversarial-error), each averaged over the $1000$ simulations, against $\mu$ values, together with corresponding quantitative formulas from equations \eqref{eq:pm}, \eqref{eq:adv-err} and \eqref{eq:s-err}.
Figure~\ref{fig:figure1}(a)-(c) shows the results for three perturbation levels of $\eta=0.05, 0.1, 0.3$, with the empirical quantities shown with symbols and the quantitative formulas shown in curves.
The plots show very good agreement between the formulas with actual empirical proportions.
\begin{figure*}[htbp]
\begin{center}
\includegraphics[height=4cm,width=0.9\textwidth]{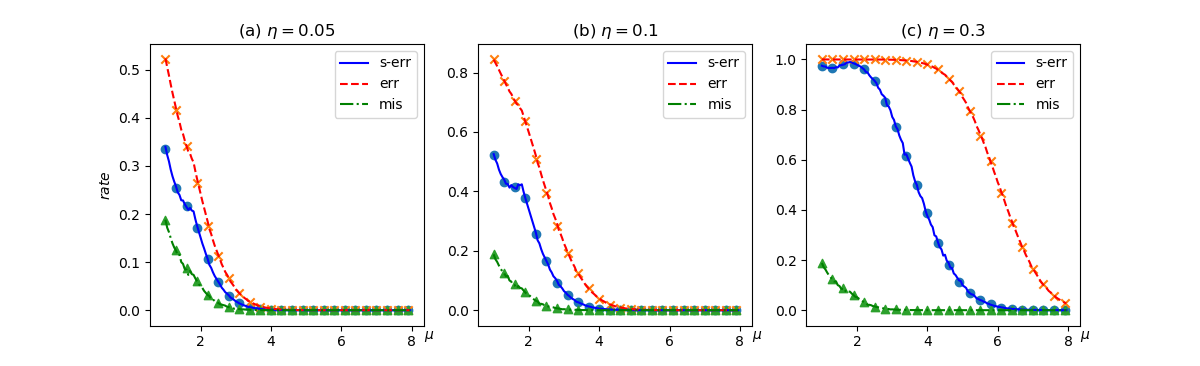}
\caption{Empirical probabilities and their theoretical values calculated from equations \eqref{eq:pm}, \eqref{eq:adv-err} and \eqref{eq:s-err}, plotted versus $\mu$. (a) $\eta=0.05$, (b) $\eta=0.1$, (c) $\eta=0.3$} 
\label{fig:figure1}
\end{center}
\end{figure*}

In our simulation, $\mu=\norm{\bs{\mu}}/1=\norm{\bs{\mu}}/\sigma$ is the SNR. Figure~\ref{fig:figure1} shows that SVM have pretty good performance in terms of misclassification rate once the SNR exceeds $2$. However, it is not robust to (strong-)adversarial attacks when $\mu=2$, and will only become robust for much larger SNR. The part of curves for $\mu<2$ have some fluctuations due to the fact that the bias term $b$ varies a lot when SNR is small. When $\mu\geq 2$, the SVM has $b\approx 0$, and we can approximate the (strong-)adversarial-error rate by dropping the bias term in \eqref{eq:adv-err} and \eqref{eq:s-err} and replace $\theta$ with its asymptotic limit as given by solving $(\theta, t)$ from the equations~\citep{Huang}:
\begin{align}
\sin^2\theta =\frac{N}{d}\int_{-\infty}^t(t-z)^2\varphi(z)\mathrm{d}z,\quad\quad
\cos\theta =\frac{N}{d}\cdot\frac{\mu}{\sigma}\int_{-\infty}^t(t-z)\varphi(z)\mathrm{d}z  \label{eq:cos}
\end{align}
where  $\varphi(z)$ is the density function of standard normal distribution. The rates plotted with these approximate formulas overlap the curves on Figure~\ref{fig:figure1} very well for the part $\mu\geq 2$. We use these formulas to study the robustness of SVM against (strong-)adversarial examples.

\begin{figure*}[htbp]
\begin{center}
\includegraphics[height=4cm,width=0.9\textwidth]{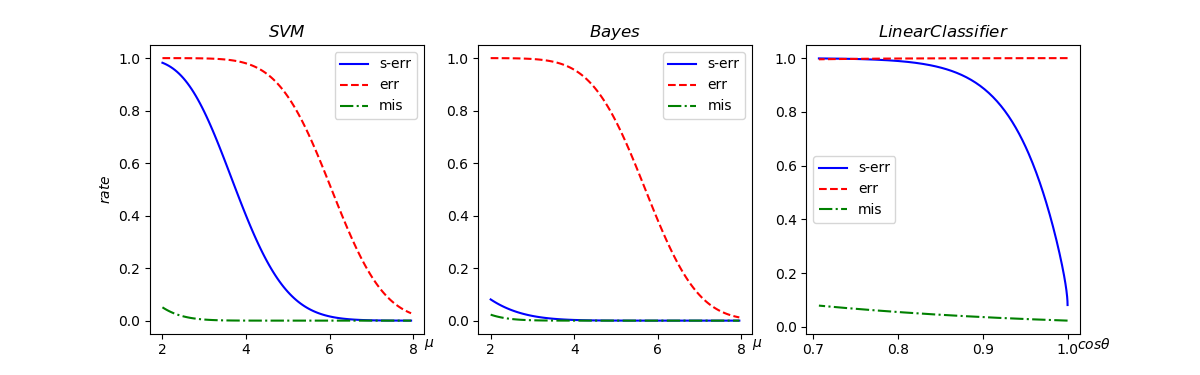}
\caption{When $\eta=0.3$, the three error rates (a) of SVM versus SNR $\mu$; (b) of Bayes classifier versus SNR $\mu$; (c) of an unbiased linear classifier versus $\cos\theta$ when $\mu=2$.} \label{fig:compare}
\end{center}
\end{figure*}

Figure~\ref{fig:compare}(a) plots the three error rates formulas of SVM when $\eta=0.3$. Figure~\ref{fig:compare}(b) plots the same rates for the Bayes classifier. These two classifiers are similar in misclassification rates and  adversarial-error rates, but are very different in strong-adversarial-error rates. For a linear classier with small bias $b' \approx 0$,  equations \eqref{eq:pm}, \eqref{eq:adv-err} and \eqref{eq:s-err} become:
\begin{align}\label{eq:rates_approx}
p_m \approx 1 - \Phi \bigg(\frac{\norm{\bs{\mu}}}{\sigma} \cos \theta\bigg),\quad p_{err} \approx 1-\Phi \bigg[\bigg(\frac{\norm{\bs{\mu}}}{\sigma}-\frac{\varepsilon}{\sigma}\bigg)\cos \theta\bigg],\quad
p_{s-err} \approx 1-\Phi\bigg[\bigg(\frac{\norm{\bs{\mu}}}{\sigma}-\frac{\delta}{\sigma}\bigg)\cos \theta - \frac{\varepsilon}{\sigma}\sin \theta\bigg]
\end{align}
Setting $\theta=0$, we get the theoretical optimal rates achieved by the ideal Bayes classifier:
\begin{align}\label{eq:rates_Bayes}
p_m^{id}=1 - \Phi \bigg(\frac{\norm{\bs{\mu}}}{\sigma}\bigg), \qquad p_{err}^{id}=1-\Phi \bigg(\frac{\norm{\bs{\mu}}}{\sigma}-\frac{\varepsilon}{\sigma}\bigg), \qquad
p_{s-err}^{id}=1-\Phi\bigg(\frac{\norm{\bs{\mu}}}{\sigma}-\frac{\delta}{\sigma}\bigg).
\end{align}
Comparing equations~\eqref{eq:rates_approx} and \eqref{eq:rates_Bayes}, between the Bayes classifier and a linear classifier with small bias, both the misclassification rate and adversarial-error rate differ by a factor of $\cos \theta$ inside the $\Phi(\cdot)$ function. However, comparing their strong-adversarial-error rates, besides the multiplicative factor $\cos \theta$, there is also an extra bias term $- \frac{\varepsilon}{\sigma}\sin \theta$ inside the $\Phi(\cdot)$ function. Since $\frac{\varepsilon}{\sigma}$ is of order $O(\sqrt{d})$, $\theta = o(1/\sqrt{d})$ is needed for the linear classifier to approach the optimal strong-adversarial-error rate. In contrast, for the misclassification rate and adversarial-error rate, only $\theta = o(1)$ is needed to approach the optimal rates.  Figure~\ref{fig:compare}(c) plots these three rates versus $\cos \theta$ when $\eta=0.3$ and $\mu=2$. We can see that misclassification rate is low for a wide range of $\cos \theta$ values while the strong-adversarial-error rate only becomes low when $\cos \theta$ is very close to one.

\subsection{Defending Against Strong-Adversarial Example Attacks}\label{sec:defense}
%
We have just seen that training a strong-adversarial-robust classifier needs stricter training requirements than those for a classifier with low misclassification rate: $\theta = o(1/\sqrt{d})$ versus $\theta = o(1)$. This is doable by incorporating some extra knowledge about the classification setting into the training. As an illustration, we show the results of using a naive method to find a sparse SVM in this case: for the SVM trained using standard method, takes ten non-zero components of $\bs{w}$ with largest absolute coefficients and set rest of components zero. The left panel of Figure~\ref{fig:sparse} plots the strong-adversarial-error rates of this sparse SVM versus original SVM. We can see that the sparse SVM achieves a low strong-adversarial-error rate very close to the optimal rate of the ideal Bayes classifier. However, the same way of finding a sparse SVM does not produce strong-adversarial-robust classifier, shown in the right panel of Figure~\ref{fig:sparse}, when the data are generated with  $\bs{\mu}_+=(\mu,\mu,...,\mu)/\sqrt{d}$ instead of $\bs{\mu}_+=(\mu,0,...,0)$. The data distributions in these two cases are equivalent with a change of coordinate systems. The sparse SVM fails in the second case since the extra knowledge incorporated into training is incorrect (sparseness only happens in the first coordinate system but not in the second coordinate system).

\begin{figure*}[htbp]
\begin{center}
\includegraphics[height=4cm,width=0.8\textwidth]{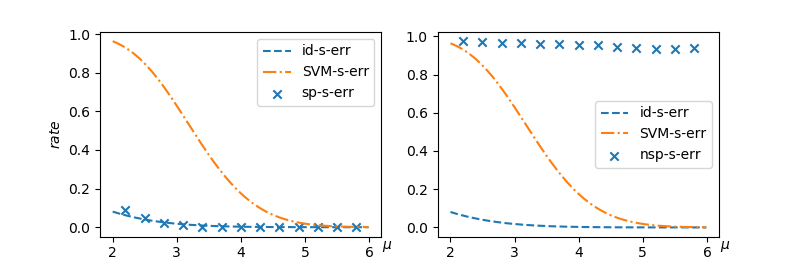}
\caption{The strong-adversarial-error rates of standard SVM ($SVM-s-err$), the sparse SVM ($sp-s-err$) and the ideal Bayes classifier ($id-s-err$). Left: $\bs{\mu}_+=(\mu,0,...,0)$; Right: $\bs{\mu}_+=(\mu,\mu,...,\mu)/\sqrt{d}$. $\eta=0.3$. }\label{fig:sparse}
\end{center}
\end{figure*}

The above exercise shows that, even when adversarial examples are unavoidable, strong-adversarial-robust linear classifiers can be found with extra structural information on the underlying problem. Notice that the sparse SVM above provides good defense by using only the knowledge of a sparse representation existence (under the coordination system) but not what the sparse representation is, with the later part learned from data by training. More generally, statisticians have noticed that SVMs are suspect to the phenomenon of data-piling: there are more data points close to the decision boundary than Gaussian mixture distribution implies. The distance-weighted discriminant~\citep{DistanceWeighted} can be used to alleviate this data-piling phenomenon, and may be used to protect against strong-adversarial examples.

\begin{figure*}[htbp]
\begin{center}
\includegraphics[height=4cm,width=0.8\textwidth]{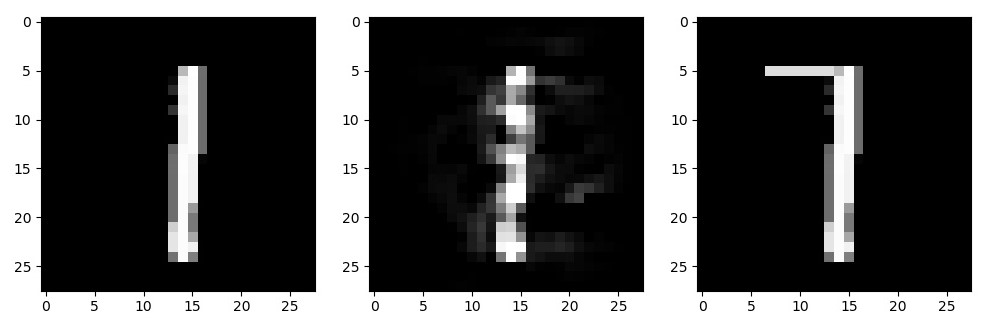}
\caption{MNIST images of '$1$': (a) the original image, (b) an adversarial example with $\varepsilon=2.34$, (c) a hand-made example with $\varepsilon=2.28$}\label{fig:MNIST}
\end{center}
\end{figure*}

For more general classification problems, the signal direction is harder to define. But the concept of adversarial versus strong-adversarial examples still applies. Figure~\ref{fig:MNIST} shows an image of '$1$' from the MNIST data set, and two images with added perturbations. 
(b) shows an adversarial example obtained by \cite{CW} algorithm with $\varepsilon=2.34$,  that is misclassified by a DNN. 
(c) shows an image we made with a similar perturbation amount $\varepsilon=2.29$.
If a classifier is adversarial-robust at level of $\varepsilon=2.34$, then it needs to classify both images (b) and (c) as  '$1$'.
However, classifying image (c) as  '$1$' clearly contradicts what a human would do, rendering the usefulness of the classifier for practical applications in doubt. Generally, we should pursue a strong-adversarial-robust classifier, not an adversarial-robust one.

\section{Discussions and Conclusions}
\label{sec:re}

In this paper, we provide clear definitions of adversarial and strong adversarial examples in the linear classification setting. Quantitative analysis shows that adversarial examples are hard to avoid but also should not be of concern in practice. Rather, we should focus on finding strong-adversarial-robust classifiers. We now consider the implications of these results on studying adversarial examples for general classifiers, and their relationship to some recent works in literature.

Recently, \citet{AdInevitable} shows that no classifier can achieve low misclassification rate and also be adversarial-robust for data distributions with bounded density on a compact region in a high-dimensional space. Our analysis does not match exactly with their impossibility statement because we are studying the Gaussian mixture case, which has positive density on the whole space. However, in spirit our results have similar implications: for the usual SNR $O(1)$ that allows low misclassification rate, generally it is impossible to be also adversarial-robust (for which a much bigger SNR $O(\sqrt{d})$ is required).

Our results, however, do show that there can be adversarial-robust classifiers under the traditional definition when the SNR is very big. 
\citet{AdvMoreData} has also shown that, for Gaussian mixture classification problem and a particular training method, the adversarial-robustness is achievable but requires more training data than simply achieving the low misclassification rate only.
Our formula indicates that useful adversarial-robust classifier do exist at the SNR level they assumed. Our study is more focused on the fundamental issue of when useful adversarial-robust classifiers exist, not which training method and what data complexity will find such a classifier. However, our formulas do indicate that an adversarial-robust classifier has to satisfy a stricter requirement than a good performing classifier. Thus either a better training method or a higher data complexity is needed for finding a useful adversarial-robust classifier, agreeing with the general theme of \citet{AdvMoreData}.

Our results on the existence of adversarial examples do not change qualitatively when using other  $\ell_p$ norm to measure the perturbation: under traditional definition, useful adversarial-robust classifier exists only when the data distribution has a very big SNR of $O(d^{min(1/p,1/2)})$ as shown in the Appendix~\ref{sec:appendix}.
For many applications where good classifiers exists (SNR of only $O(1)$ ensures this), we can not pursue adversarial-robust classifier under the traditional adversarial example definition~\ref{def:adv}.
The current defense strategies based on such adversarial example definition likely will still be suspect to more sophisticated adversarial attacks. For certifiable adversarial-robust classifiers~\citep{DLResistance, DistributionalRobustness}, the robustness is achieved only for the perturbation amount $\varepsilon$ high enough so that they differ from human in classifying images like those in Figure~\ref{fig:example}(c) and Figure~\ref{fig:MNIST}(c).
Thus a paradigm change is needed: we should train a classifier to be strong-adversarial-robust rather than adversarial-robust.

While the signal direction is obvious in the linear classification, the signal direction and the definition of strong-adversarial examples in general classification warrants further study.
The signal direction in the linear classification here is the direction where the likelihood ratio of the two classes changes most rapidly.
One reasonable extension is to define the signal direction at any data vector $\bs{x}$  as the gradient direction of the likelihood ratio at $\bs{x}$.
Then similar to definition~\ref{def:s-adv}, the strong-adversarial example for general classifier also restrict the change along this signal direction to the amount $\delta$.
The strong-adversarial-robust classifiers therefore are likely to be very close to the Bayes classifier.
Some recent works have attempted training DNN to be close to the Bayes classifier: \citet{advnn} uses a nearest neighbors method, and \citet{AdvRobustMNIST} applies the generative model techniques. In particular, \citet{AdvRobustMNIST} applied their method on MNIST dataset, and when applying a specifically designed attack on such a trained DNN, the adversarial examples found are semantically meaningful for humans. That is, these adversarial examples are adversarial in traditional definition but likely not strong-adversarial.
The new strong-adversarial examples framework can allow theoretical quantification of the robustness for these training methods.
The analysis of strong-adversarial-robustness for general classifiers such as DNN can provide a new research direction on how to defend against realistic adversarial attacks.

\citestyle{IEEEtran}
\bibliographystyle{IEEEtranN}
\bibliography{AdRef}
\section{Appendix}\label{sec:appendix}
\subsection{Proof of Lemma~1}\label{sec:app}

\begin{lem}\label{lem:set} The defining sets for $\varepsilon$-adversarial and $(\varepsilon,\delta)$-strong adversarial examples are given by:
\begin{equation}\label{eq:set-semi}
\Omega_\varepsilon = \Omega(\varepsilon \bs{v}_0)\cup \Omega(-\varepsilon\bs{v}_0); \quad \Omega_{\varepsilon,\delta} = \Omega(\bs{u}_2)\cup\Omega(-\bs{u}_2)
\end{equation}
where $\bs{u}_2=\beta\bs{\mu}_0+\sqrt{\varepsilon^2-\beta^2}\bs{n}_0, \beta=\min(\varepsilon\cos\theta,\delta)$.
\end{lem}

\begin{proof}

{\bf Proof of the adversarial defining set formula.}
Since it is obvious from the definition that $\Omega_\varepsilon=\bigcup_{\norm{\bs{v}}\leq\varepsilon}\Omega(\bs{v}) \supseteq \Omega(\varepsilon \bs{v}_0)\cup \Omega(-\varepsilon\bs{v}_0)$, we only need to show that $\Omega_\varepsilon \subseteq \Omega(\varepsilon \bs{v}_0)\cup \Omega(-\varepsilon\bs{v}_0)$. That is, for any data point $\bs{x} \in \Omega_\varepsilon$, either $\bs{x}+\varepsilon \bs{v}_0$ or $\bs{x}-\varepsilon \bs{v}_0$ changes its classification.

\noindent We now claim that the last statement is equivalent to that $\varepsilon \bs{v}_0$ is the solution to the optimization problem:
\be\label{eq:opt-sadv}
\max\bs{w}\cdot\bs{v}, \qquad \bs{v} \in D_\varepsilon=\{\bs{v}\in\mathbb{R}^d:\norm{\bs{v}}\leq\varepsilon\}.
\ee
To see this, if $\varepsilon \bs{v}_0$ is the solution, then $\bs{w}\cdot\bs{v} \leq \bs{w}\cdot (\varepsilon \bs{v}_0) = \varepsilon \norm{\bs{w}}$ for all $\bs{v} \in D_1$. Now for a $\bs{x}$ classified into the '$-$' class and $\bs{x} \in \Omega_\varepsilon$, then $\bs{w}\cdot\bs{x}+b<0$ and $\bs{w}\cdot(\bs{x}+\bs{v})+b>0$. Hence
$$
\bs{w}\cdot(\bs{x}+\varepsilon \bs{v}_0)+b \geq \bs{w}\cdot\bs{x} + \bs{w}\cdot\bs{v} +b >0,
$$
that is, $\bs{x}+\varepsilon \bs{v}_0$ is misclassified into the '$+$' class thus $\bs{x} \in \Omega(\varepsilon \bs{v}_0)$. By symmetry, $- \varepsilon {\bs{v}_0}$ is the solution to $\min\bs{w}\cdot\bs{v}$ when $\bs{v} \in D_1$, and hence $- \varepsilon \norm{\bs{w}} \leq \bs{w}\cdot\bs{v}$ also for all $\bs{v} \in D_1$. Hence for a $\bs{x}$ classified into the '$+$' class and $\bs{x} \in \Omega_\varepsilon$, similarly we have that $\bs{x}-\varepsilon \bs{v}_0$ is misclassified into the '$-$' class thus $\bs{x} \in \Omega(-\varepsilon \bs{v}_0)$.

\noindent Finally, $\varepsilon \bs{v}_0$ is indeed the solution to \eqref{eq:opt-sadv} due to the Cauchy-Schwartz inequality $\bs{w}\cdot\bs{v}\leq\norm{\bs{w}}\norm{\bs{v}} \leq \norm{\bs{w}} \varepsilon$. The first equality holds if and only if $\bs{v}$ is along the same direction of $\bs{w}$, thus $\bs{v}= c \bs{v}_0$. The second equality holds if and only if $\norm{\bs{v}}=\varepsilon$, thus $\bs{v}= \varepsilon \bs{v}_0$. This finishes the proof for $\Omega_\varepsilon = \Omega(\varepsilon \bs{v}_0)\cup \Omega(-\varepsilon\bs{v}_0)$.

\noindent{\bf Proof of the strong adversarial defining set formula.}
The proof follows exactly the outline of the adversarial case proof above. Only now we need to prove that $\bs{u}_2$ is the solution to the optimization problem
\be\label{eq:opt-adv}
\max\bs{w}\cdot\bs{v}, \qquad \bs{v} \in D_{\varepsilon,\delta}=\{\bs{v}\in\mathbb{R}^d:\norm{\bs{v}}\leq\varepsilon,|\bs{v}\cdot\bs{\mu}_0|\leq\delta\}.
\ee
We can decompose $\bs{w}$ as $\bs{w}=(\bs{w}\cdot\bs{\mu}_0)\bs{\mu}_0+(\bs{w}\cdot\bs{n}_0)\bs{n}_0$, accordingly, $\bs{v}$ can be decomposed as $\bs{v}=(\bs{v}\cdot\bs{\mu}_0)\bs{\mu}_0+(\bs{v}\cdot\bs{n}_0)\bs{n}_0+(\bs{v}\cdot\bs{m}_0)\bs{m}_0$, where $\bs{m}_0$ is the unit normal vector of the plane spanned by $\bs{\mu}_0$ and $\bs{w}$, therefore
\begin{equation}\label{eq:wv}
\bs{w}\cdot\bs{v}=(\bs{w}\cdot\bs{\mu}_0)(\bs{v}\cdot\bs{\mu}_0)+(\bs{w}\cdot\bs{n}_0)(\bs{v}\cdot\bs{n}_0)=\cos \theta (\bs{v}\cdot\bs{\mu}_0)+\sin \theta (\bs{v}\cdot\bs{n}_0):=x \cos \theta + y \sin \theta.
\end{equation}
The optimization problems becomes to maximize $x \cos \theta + y \sin \theta$ in \eqref{eq:wv} under the constraints $x^2+y^2=\varepsilon^2-(\bs{v}\cdot\bs{m}_0)^2, |x|\leq\delta$. This is a linear programming setup, it is easy to see that first we must have $\bs{v}\cdot\bs{m}_0=0$ to reach maximum. Then the solution is either at the corner $(x,y)=(\delta,\sqrt{\varepsilon^2-\delta^2})$ or at the tangent point $(x,y)=\varepsilon(\cos \theta, \sin \theta)$ as in semi-adversarial case. If $\varepsilon \cos \theta < \delta$, the solution is at the tangent point $(x,y)=\varepsilon(\cos \theta, \sin \theta)$. Otherwise, the solution is at the corner $(x,y)=(\delta,\sqrt{\varepsilon^2-\delta^2})$. Combining the two cases, we arrive at the formula for $\bs{u}_2$ under equation~\eqref{eq:set-semi}.

\end{proof}

\subsection{$\ell_p$-Adversarial and $\ell_p$-Strong-Adversarial Rates}\label{sec: lp}

\noindent In literature, the adversarial examples have been studied under different norms. Here we extend the analysis in main text to the general $\ell_p$ norms with $p\in [1,\infty]$\footnote{We did not consider $p\in [0,1)$ because in this case, $d_p$ is not a metric, although practically, $\ell_0$ is considered.}.
That is, we use the distance metric $d_p(\bs{x},\bs{y})=\norm{\bs{x}-\bs{y}}_p$. 
Also, we denote $\ell_q$ as the dual of $\ell_p$, i.e., $1/p+1/q=1$.

\noindent Therefore the classical adversarial examples definition becomes the following.
\begin{defi}\label{def:lp-adv}
Given a classifier $C$, an $\varepsilon$-$\ell_p$-adversarial example of a data vector $\bs{x}$  is another data vector $\bs{x}'$ such that $d_p(\bs{x},\bs{x}') \leq \varepsilon$ but $C(\bs{x})\neq C(\bs{x}')$.
\end{defi}

\noindent As before, we restrict the perturbation amount along the signal direction $\bs{\mu}_0$ to $\delta$ for strong-adversarial examples.
\begin{defi}\label{def:lp-s-adv}
Given a classifier $C$, an $(\varepsilon,\delta)$-$\ell_p$-strong-adversarial example of a data vector $\bs{x}$  is another data vector $\bs{x}'$ such that $d_p(\bs{x},\bs{x}') \leq \varepsilon$ and $|(\bs{x}-\bs{x}')\cdot\bs{\mu}_0|\leq \delta$ but $C(\bs{x})\neq C(\bs{x}')$.
\end{defi}

\subsection{$\ell_p$-Adversarial Rate and Existence of $\ell_p$-Adversarial-Robust Classifiers}\label{sec:lp-adv}

\noindent The analysis follows the same outline as the analysis for the $\ell_2$ norm case. We first characterize the defining set $\Omega_{\varepsilon|p}=\{\bs{x}:\bs{x}\ \mbox{has\ an\ } \varepsilon-\ell_p\mbox{-adversarial\ example}\}$. 

\begin{lem}\label{lem:lp-adv-set}
The defining sets for $\varepsilon$-$\ell_p$-adversarial examples is given by:
\begin{equation}\label{eq:lp-adv-set}
\Omega_{\varepsilon |p}=\Omega(\varepsilon\bs{v}_{0|p})\cup\Omega(-\varepsilon\bs{v}_{0|p})
\end{equation}
where $\bs{v}_{0|p}$ is the $d$-dimensional vector with component $(\bs{v}_{0|p})_i=\sgn (w_i)\cdot (|w_i|/\norm{\bs{w}}_q)^{q-1}$. 
\end{lem}
\noindent Here $\sgn$ denotes the sign function. That is, $\sgn(x)=1$ for $x>0$; $\sgn(x)=-1$ for $x<0$ and $\sgn(0)=0$. 

\noindent Furthermore, we denote the $p$-th power of a vector $\bs{v}=(v_1,...,v_d)$ as taking the power component-wise. That is, $(\bs{v}^p)_i = \sgn (v_i)\cdot |v_i|^{p}$. 
Then the above $\bs{v}_{0|p}$ can be rewritten as $\bs{v}_{0|p}=(\bs{w}/\norm{\bs{w}}_q)^{q-1}$.

\begin{proof}
The proof is similar to the proof of Lemma~\ref{lem:set}. Following the derivations there, we only need to show that $\bs{v}_0=(\bs{w}/\norm{\bs{w}}_q)^{q-1}$ is the solution to the optimization problem:
\begin{equation}\label{opt-lp-adv}
\max |\bs{w}\cdot\bs{v}|, \qquad \bs{v}\in D_{\varepsilon|p}=\{ \bs{v}\in\mathbb{R}^d:\norm{\bs{v}}_p\leq\varepsilon\}.
\end{equation}
By Holder's inequality, we have $|\bs{w}\cdot\bs{v}| \leq \norm{\bs{w}}_q\norm{\bs{v}}_p \le \varepsilon\norm{\bs{w}}_q$.
 
\noindent For the first ``$\leq$'' to be ``$=$'', $\bs{v}^p$ has to be proportional to $\bs{w}^q$. That is, for some constant $c$, $\bs{v}= c \bs{w}^{q/p} = c \bs{w}^{q-1}$. For the second  ``$\leq$'' to be ``$=$'', we need $\varepsilon = \norm{\bs{v}}_p$. That is, 
$$
\varepsilon^p = \norm{\bs{v}}_p^p = c^p \sum_{i=1}^d |v_i|^p = c^p \sum_{i=1}^d (|w_i|^{q/p})^p = c^p \sum_{i=1}^d (|w_i|^{q}) = c^p \norm{\bs{w}}_q^q.
$$
Hence we have $\varepsilon=c  \norm{\bs{w}}_q^{q/p} = c  \norm{\bs{w}}_q^{q-1}$, and thus $c = \varepsilon \norm{\bs{w}}_q^{1-q}$. 
Plug $c$ into $\bs{v}= c \bs{w}^{q-1}$, we get $\bs{v}=\varepsilon (\bs{w}/\norm{\bs{w}}_q)^{q-1}= \varepsilon \bs{v}_{0|p}$. This is the solution to the optimization problem. Hence arguments similar to those for the proof of Lemma~\ref{lem:set} above show that the equation~\eqref{eq:lp-adv-set} gives the defining set here.
\end{proof}

\noindent With the characterization lemma~\ref{lem:lp-adv-set}, we can then compute the adversarial rate as before. Note that the misclassification rate has nothing to do with the perturbation for adversarial examples. Thus regardless of which  $\ell_p$ norm is used to measure the perturbation, the misclassification rate is still given by the same formula as before.
\begin{align}\label{eq:app_pm}
p_m  = 1 - 0.5\bigg[\Phi\bigg(\frac{\bs{w}\cdot\bs{\mu}+b'}{\norm{\bs{w}}\sigma}\bigg) + \Phi\bigg(\frac{\bs{w}\cdot\bs{\mu}-b'}{\norm{\bs{w}}\sigma}\bigg) \bigg].
\end{align}

\noindent The calculation of $\ell_p$-adversarial rate follows $\ell_2$-adversarial rate calculation exactly, except that the term $\varepsilon\norm{\bs{w}}_2$ is replaced by $\bs{w}\cdot\varepsilon\bs{v}_{0|p}=\varepsilon\norm{\bs{w}}_q$. Therefore, we have the following result.

\begin{thm}\label{thm:lp-adv-rate}
The overall $\ell_p$-adversarial rate of a linear classifier for the balanced Gaussian mixture data is
\begin{equation}\label{eq:lp-adv-rate}
p_{adv|p} =1-p_m-0.5\bigg[\Phi\bigg(\frac{\bs{w}\cdot\bs{\mu}+b'}{\norm{\bs{w}}_2\sigma}-\frac{\norm{\bs{w}}_q}{\norm{\bs{w}}_2}\frac{\varepsilon}{\sigma}\bigg)+\Phi\bigg(\frac{\bs{w}\cdot\bs{\mu}-b'}{\norm{\bs{w}}_2\sigma}-\frac{\norm{\bs{w}}_q}{\norm{\bs{w}}_2}\frac{\varepsilon}{\sigma}\bigg)\bigg].
\end{equation}
\end{thm}

\noindent Now we have the $\ell_p$-adversarial error formula as the following. 
\begin{align}\label{eq:lp-adv-err}
p_{err|p} &= p_{adv|p} + p_m = 1-0.5\bigg[\Phi\bigg(\frac{\bs{w}\cdot\bs{\mu}+b'}{\norm{\bs{w}}_2\sigma}-\frac{\norm{\bs{w}}_q}{\norm{\bs{w}}_2}\frac{\varepsilon}{\sigma}\bigg)+\Phi\bigg(\frac{\bs{w}\cdot\bs{\mu}-b'}{\norm{\bs{w}}_2\sigma}-\frac{\norm{\bs{w}}_q}{\norm{\bs{w}}_2}\frac{\varepsilon}{\sigma}\bigg)\bigg]\nonumber\\
& \geq 1 - \Phi\bigg(\frac{\bs{w}\cdot\bs{\mu}}{\norm{\bs{w}}_2\sigma}-\frac{\norm{\bs{w}}_q}{\norm{\bs{w}}_2}\frac{\varepsilon}{\sigma}\bigg) = 1 - \Phi\bigg(\frac{\norm{\bs{\mu}}_2}{\sigma}\cos\theta-\frac{\norm{\bs{w}}_q}{\norm{\bs{w}}_2}\frac{\varepsilon}{\sigma}\bigg)
\end{align}

\noindent Corresponding to the discussions in the main text, a useful classifier only requires a signal-noise ratio (SNR) of $\norm{\mu}_2/\sigma=O(1)$ due to equation~\eqref{eq:app_pm}. 

\noindent In contrast, due to equation~\eqref{eq:lp-adv-err}, a necessary condition for the existence of a $\ell_p$-adversarial-robust classifier is 
\be\label{eq:cond}
\frac{\norm{\bs{\mu}}_2}{\sigma}-\frac{\norm{\bs{w}}_q}{\norm{\bs{w}}_2}\frac{\varepsilon}{\sigma}=O(1).
\ee

\noindent We now investigate what order of SNR $\norm{\mu}_2/\sigma$ is needed to make \eqref{eq:cond} hold.

\noindent First, we have to find the practical relevant order of $\varepsilon$ needs to be studied. The following lemma about the average $\ell_p$-norm of Gaussian noise will provide the guideline. 

\begin{lem}\label{lem:lp-noise}
Let the random vector $\bs{x}=(x_1,...,x_d)$ follows the $d$-dimensional Gaussian distribution $N(0,\sigma^2I_d)$. Then
\begin{align}\label{eq:lp-noise}
\begin{cases} E[\norm{x}_p^p]= m_p d \sigma^p,  \qquad & p\in [1, \infty),\\
 E[\norm{x}_p]=O(\sqrt{\log d}\sigma), \qquad & p=\infty,\end{cases}
\end{align}
where $m_p$ denotes the $p$-th moment of the standard Gaussian distribution. 
\end{lem}
\begin{proof}
For $p\in [1, \infty)$, 
\begin{align}
E[\norm{x}_p^p]= E[\sum_{i=1}^d |x_i|^p] = \sum_{i=1}^d E[|x_i|^p] = d E[|x_1|^p] = d \sigma^p m_p. 
\end{align}
The $\ell_\infty$ result follows directly from the large deviation formula obtained by \citet{Linftynoise}.
\end{proof}

\noindent In the Gaussian mixture data, Lemma~\ref{lem:lp-noise} states that the average $\ell_p$ noise is $d^{1/p}\sigma m_p^{1/p}$. Therefore, for an $\eta<1$, a perturbation amount of $\varepsilon=\eta d^{1/p}\sigma m_p^{1/p}$ will be smaller than the average noise thus hard to distinguish from noise (unless it concentrates in the signal direction). Thus any practical relevant defense needs to be robust at the minimum against perturbations of order $\varepsilon=O(d^{1/p}\sigma)$. For the $\ell_\infty$, the defense needs to be robust at the minimum against perturbations of order $\varepsilon=O(\sqrt{\log d}\sigma)$.

\noindent Plug-in these $\varepsilon$ orders into the necessary condition~\eqref{eq:cond}, the existence of a $\ell_p$-adversarial-robust classifier requires at least $\frac{\norm{\bs{\mu}}_2}{\sigma} = O(\frac{\norm{\bs{w}}_q}{\norm{\bs{w}}_2}d^{1/p})$ for $p\in [1, \infty)$; and it requires at least $\frac{\norm{\bs{\mu}}_2}{\sigma} = O(\frac{\norm{\bs{w}}_q}{\norm{\bs{w}}_2}\sqrt{\log d})$ for $p=\infty$.

\noindent Next we use the norm comparison inequality to find these orders. For any $\bs{w}\in\mathbb{R}^d$ and any $0<r<s<\infty$, we have
\begin{align}\label{eq:norm-ineq}
\norm{\bs{w}}_s\leq \norm{\bs{w}}_r\leq d^{1/r-1/s}\norm{\bs{w}}_s.
\end{align}

\noindent {\bf (A)} For $1\leq p<2$, we have $2<q\leq\infty$. Using $r=2$ and $s=q$ in \eqref{eq:norm-ineq}, we get 
$$
d^{1/q-1/2}\leq \frac{\norm{\bs{w}}_q}{\norm{\bs{w}}_2}\leq 1.
$$
Plug this lower bound into the required order $\frac{\norm{\mu}_2}{\sigma} = O(\frac{\norm{\bs{w}}_q}{\norm{\bs{w}}_2}d^{1/p})$, the existence of a $\ell_p$-adversarial-robust classifier requires SNR  of at least
$$
\frac{\norm{\bs{\mu}}_2}{\sigma} = O(d^{1/q-1/2}d^{1/p}) = O(d^{1/p+1/q-1/2}) = O(d^{1/2}).
$$

\noindent {\bf (B)} For $2 < p<\infty$, then $q<2<\infty$. Using $r=q$ and $s=2$ in \eqref{eq:norm-ineq}, we get  
\[
1\leq \frac{\norm{\bs{w}}_q}{\norm{\bs{w}}_2}\leq d^{1/q-1/2}.
\]
Thus the existence of a $\ell_p$-adversarial-robust classifier requires SNR  of at least
$$
\frac{\norm{\bs{\mu}}_2}{\sigma} = O(1 \cdot d^{1/p}) = O(d^{1/p}).
$$

\noindent {\bf (C)} When $p=\infty$, then $q=1$. Using $r=q$ and $s=2$ in \eqref{eq:norm-ineq}, we get
\[
1\leq \frac{\norm{\bs{w}}_1}{\norm{\bs{w}}_2}\leq d^{1/q-1/2}.
\]
Thus the existence of a $\ell_p$-adversarial-robust classifier requires SNR  of at least
$$
\frac{\norm{\bs{\mu}}_2}{\sigma} = O(1 \cdot \sqrt{\log d}) = O(\sqrt{\log d}).
$$

\noindent We summarize the results for cases (A), (B) and (C)  into the following theorem. 
\begin{thm}\label{thm:adv-inev}
For linear classification of balanced Gaussian mixture data, the existence of a $\ell_p$-adversarial-robust classifier requires SNR  of at least
\begin{align}
\frac{\norm{\bs{\mu}}_2}{\sigma}=\begin{cases}
O(d^{\min(1/p, 1/2)})\qquad& p\in[1, \infty),\\
O(\sqrt{\log d})\qquad &p=\infty.
\end{cases}
\end{align}
\end{thm}

\noindent Theorem ~\ref{thm:adv-inev} shows that the required SNR magnitude for $\ell_p$-adversarial-robustness differs for different $p$. The $\ell_p$-adversarial-robustness is hardest to achieve for $1 \le p \le 2$ since the required SNR $O(\sqrt{d})$ is highest in these cases. The $\ell_\infty$-adversarial-robustness has the smallest SNR requirement, thus easiest to achieve. This agrees with the observation by \citet{AdvRobustMNIST}: the $\ell_\infty$ robust classifier in \citet{DLResistance} is still highly susceptible to $\ell_2$ attack.

\subsection{$\ell_p$-Strong-Adversarial Rate and Existence of $\ell_p$-Strong-Adversarial-Robust Classifiers}\label{sec:lp-s-adv}

Following the same derivations before, we have the following lemma for the defining set $\Omega_{\varepsilon,\delta |p}=\{\bs{x}: \bs{x}\ \mbox{has\ an\ }(\varepsilon,\delta)-\ell_p\mbox{-strong-adversarial\ example}\}$.

\begin{lem}\label{lem:lp-s-adv-set}
The defining set for $(\varepsilon,\delta)-\ell_p$-strong-adversarial examples is given by:
\begin{equation}\label{eq:lp-s-adv-set}
\Omega_{\varepsilon,\delta |p}=\Omega(\bs{u}_p)\cup\Omega(-\bs{u}_p)
\end{equation}
where $\bs{u}_p$ is the solution to the optimization problem:
\begin{equation}\label{opt-lp-s-adv}
\max |\bs{w}\cdot\bs{v}|, \qquad \bs{v}\in D_{\varepsilon,\delta|p}=\{ \bs{v}\in\mathbb{R}^d:\norm{\bs{v}}_p\leq\varepsilon, |\bs{v}\cdot\bs{\mu}_0|\leq\delta\}\subset D_{\varepsilon|p}.
\end{equation}
\end{lem}
\noindent Notice the optimization problem of $\max |\bs{w}\cdot\bs{v}|$ is a linear programming problem, and the feasible region  $D_{\varepsilon,\delta|p}$ is a convex region. Therefore the solution $\bs{u}_p$ does exist.

\noindent Now replace the term $\varepsilon\norm{\bs{w}}_2$ by $\bs{w}\cdot\bs{u}_{p}$ in the previous derivations of $(\varepsilon,\delta)$-strong-adversarial rate using $\ell_2$ norm, we get the following Theorem.
\begin{thm}\label{thm:lp-s-adv-rate}
The overall $(\varepsilon,\delta)-\ell_p$-strong-adversarial rate of a linear classifier for the balanced Gaussian mixture data is
\begin{equation}\label{eq:lp-s-adv-rate}
p_{s-adv|p} =1-p_m-0.5\bigg[\Phi\bigg(\frac{\bs{w}\cdot\bs{\mu}+b'}{\norm{\bs{w}}_2\sigma}-\frac{\bs{w}\cdot\bs{u}_p}{\norm{\bs{w}}_2\sigma}\bigg)+\Phi\bigg(\frac{\bs{w}\cdot\bs{\mu}-b'}{\norm{\bs{w}}_2\sigma}-\frac{\bs{w}\cdot\bs{u}_p}{\norm{\bs{w}}_2\sigma}\bigg)\bigg]
\end{equation}
\end{thm}

\noindent We now try to find an SNR order that allows $\ell_p$-strong-adversarial-robustness by applying formula~\eqref{eq:lp-s-adv-rate} to the Bayes classifier whose $\bs{w}=\bs{\mu}_0$ and $b'=0$. In this case, the solution to the optimization problem~\eqref{opt-lp-s-adv} becomes $\bs{u}_p=\delta\bs{\mu}_0$. Thus using formula~\eqref{eq:lp-s-adv-rate} we can get the $(\varepsilon,\delta)-\ell_p$-strong-adversarial-error rate for the Bayes classifier as
\begin{equation}\label{eq:Bayes-s-err-t}
p_{s-err|p} = 1-\Phi\bigg(\frac{\norm{\bs{\mu}}_2-\delta}{\sigma}\bigg).
\end{equation}
Since practical relevant $\delta$ can not exceed $\sigma$ (in that case, no classifier can work as at least half of all data vectors will be perturbed into another class),  SNR $\frac{\norm{\bs{\mu}}_2}{\sigma}$ of order $O(1)$ can result in a useful $\ell_p$-strong-adversarial-robust classifier. This agrees with the conclusions in the main text about the existence of $\ell_2$-strong-adversarial-robust classifiers.

\end{document}